\documentclass[letterpaper, 10 pt, conference]{ieeeconf}  

\IEEEoverridecommandlockouts                              

\usepackage{graphics} 
\usepackage{amsmath} 

\usepackage{mathtools,amsthm}

\usepackage{amssymb}  
\usepackage{graphicx}
\usepackage[table]{xcolor}
\usepackage{placeins}
\usepackage[font={small}]{caption}
\usepackage{subcaption}
\usepackage{verbatim}
\usepackage{hyperref}

\DeclareMathOperator*{\argmin}{arg\ min.}

\newcommand{\mytheorem}[2]{%
\newtheorem{t#2}{#1}%
\newenvironment{#2}{\begin{t#2}}{\end{t#2}}}

\newcommand{\myremark}[2]{%
\newtheorem{t#2}{#1}[section]%
\newenvironment{#2}{\begin{t#2}}{\end{t#2}}}

\theoremstyle{plain}

\mytheorem{Theorem}{theorem}
\mytheorem{Lemma}{lemma}
\mytheorem{Proposition}{proposition}
\mytheorem{Corollary}{corollary}
\mytheorem{Assumption}{assumption}
\mytheorem{Definition}{definition}
\mytheorem{Property}{property}
\myremark{Remark}{ownremark}

\title{\LARGE \bf
Model Based In Situ Calibration of Six Axis Force Torque Sensors
}

\author{Francisco Javier Andrade Chavez$^{1,2}$, Silvio Traversaro$^{2}$, Daniele Pucci$^{2}$ and Francesco Nori$^{2}$
\thanks{This paper was supported by the FP7 EU projects CoDyCo (No. 600716
ICT 2011.2.1 Cognitive Systems and Robotics) and Koroibot (No. 611909
ICT-2013.2.1 Cognitive Systems and Robotics).}
\thanks{$^{1}$ Francisco Javier Andrade Chavez with CVU 468287 receives support from the National Council of Science and Technology in Mexico, CONACYT}
\thanks{$^{2}$ All authors belong to the Istituto Italiano di Tecnologia, iCub Facility department, Genova, Italy.
       Emails: {\tt\small ing.andrade.francisco@gmail.com}, 
         {\tt\small silvio.traversaro@iit.it}, 
          {\tt\small francesco.nori@iit.it}, 
           {\tt\small daniele.pucci@iit.it}}
}

\begin{document}

\maketitle
\thispagestyle{empty}
\pagestyle{empty}

\begin{abstract}
This paper proposes and validates an \textit{in situ} calibration method to calibrate six axis force torque (F/T) sensors once they are mounted on the system. This procedure takes advantage of the knowledge of the model of the robot to generate the expected wrenches of the sensors during some arbitrary motions. It then uses this information to train and validate new calibration matrices, taking into account the calibration matrix obtained with a classical Workbench calibration. 
The proposed calibration algorithm is validated on the F/T sensors mounted on the iCub humanoid robot legs. 

\end{abstract}

\section{INTRODUCTION}
Six axis force torque (F/T) sensors have been used in robotics systems since the 1970's \cite{watson1975pedestal}. Around the same time  research on force control began \cite{whitney1987historical}.
F/T sensing has become a important sensing capability which can be highly exploited in robotics since is an essential knowledge for regulating contact forces and torques.
 Although strain gauge-based sensing technology has been widely used in industrial robots, its practical use in humanoids robots have been limited by the experimental evidence that installing the F/T sensors into complex structures such as Humanoids seems to change the measurements returned by the sensors making them unreliable. As explained in \cite{DRC-what-happened}, during the last DARPA Robotics Challenge most of the teams could not take advantage of having such sensors. The Boston Dynamics ATLAS' Six Axis F/T sensors were not used due to the bad quality of sensors measurements, to the point that the IHMC and MIT teams used the F/T sensors only as binary contact sensor.
 In standard operating conditions, relevant changes in the calibration matrix may occur in months. F/T sensors are recommended to be calibrated at least once a year as stated by ATI \cite{atimanual} and  Weiss Robotics \cite{kms40manual} which are some of the leading companies for F/T sensors. The calibration procedure is done by the manufacturer company which implies that the sensor must be unmounted and sent back to them and then mounted again. It has been noted during different experiments that the reliability of the measurements changes after mounting the sensor on the robot, even in recently calibrated sensors. The knowledge of external forces on the robot can be used for more advanced control strategies, which makes having a properly calibrated F/T sensor essential for allowing robots to perform more complex actions.
 
 Most of the six axis F/T sensors available on the market are based on silicon or metallic strain gauges technologies, even if alternative technologies are starting to be adopted \cite{tar2011development}. 
 
 The commonly used model for predicting the force-torque from the raw strain gauges measurements of the sensor is an affine model. This model is sufficiently accurate since these sensors are mechanically designed and mounted so that the strain deformation is (locally) linear with respect to the applied forces and torques. Then, the calibration of the sensor aims at determining the two components of this model, i.e. a six-by-six matrix and a six element vector. These two components are usually referred to as the sensor's \emph{calibration matrix} and \emph{offset}, respectively. Preponderant changes in the sensor's offset can occur in hours, however, and this in general requires to estimate the offset before using the  sensor. 
 
The typical calibration procedure considers first identifying the offset when no load is applied on the sensor and then carefully place some weights in specific positions to have well known gravitational forces and torques in order to span the space of the sensor. The methods for obtaining the calibration matrix have been thoroughly studied and, although many methods exist, least squares remains the most popular \cite{braun2011}. For simplifying the time consuming procedure of careful load placing some specialized structures have been designed \cite{uchiyama1991systematic,watson1975pedestal}. In other cases a previously calibrated sensor is used as reference ~\cite{faber2012force}~\cite{oddo2007}. This has the disadvantage of depending on the availability of another sensor which is not always the case.

The difference that has been observed on a six axis F/T sensor after being mounted has motivated the search for \textit{in situ} calibration methods \cite{InSituAcc}. Among other advantages, these methods allow to perform the calibration in the sensor's final destination avoiding possible modification of the calibration matrix that arise from mounting and removing the sensors from its working structure. To the best of our knowledge, the first \emph{in situ} calibration method for force-torque sensors was proposed in \cite{shimanoroth}. But this method  exploits the topology of a specific kind of manipulators, which are equipped with joint torque sensors then leveraged during the estimation. 
Another in situ calibration technique for force-torque sensors can be found in \cite{roozbahani2013novel}. But the use of supplementary already-calibrated force-torque/pressure sensors impairs this technique for the reasons we have discussed before. 

In our previous work \cite{InSituAcc} a F/T sensor was calibrated \emph{in situ} by assuming that a single rigid body equipped with an accelerometer was attached to the F/T sensor. While we assumed that the inertial parameters (mass, center of mass, 3D inertia tensor) of the attached rigid body were unknown, nevertheless we assumed that a set of additional masses of known mass was attached to the F/T sensor load in the various experiments. Even the (limited) assumptions of \cite{InSituAcc} complicated a lot the use of the introduced techniques. In particular the need for knowing a-priori the accelerometer orientation w.r.t. the F/T sensor frame and the assumption that only a rigid body was attached to the F/T sensor complicated the use of such techniques in the case of F/T embedded in the robot structure for performing joint torques estimation \cite{Fumagalli2012}. To overcome this limitations, in this paper we assume that the inertial parameters of robot links are known. While this may seem a rather bold assumption, it is possible if the inertial parameters obtained from the Computer Aided Manufacturing (CAD) model of the robot are validated by weighting experiments on the individual robot links, as was our case. 

There are two main contribution in this paper. The first one is to formulate the calibration problem by decoupling the offset estimation problem from the calibration matrix estimation problem, enabling the use of multiple datasets with multiple unknown offsets that share the same calibration matrix. The second one is to cast the calibration matrix estimation problem as \emph{regularized least square} problem, in which the regularization takes into account the information known from a previous available calibration matrix.

The proposed algorithms are validated by calibrating a six axis F/T sensor found in the right leg of an iCub humanoid robot. Some of the calibration and validation datasets come from real world scenarios in which the iCub is switching from two feet balancing to one foot balancing.

The paper is organized as follows. Section \ref{sec:problem} describes the formulation of the problem. Section \ref{sec:method} describes the strategies used to do the \textit{model based in situ} calibration. Section \ref{sec:experiments} describes the characteristics of the datasets used for training and the validation procedures used. Section \ref{sec:results} shows  the results of both validation procedures and \ref{sec:conclusions} states the insights obtained through the experiments.


\begin{figure}
\begin{minipage}[c]{0.29\textwidth}
    \caption{
      Location of the six axis F/T sensors mounted on the iCub. The F/T sensors are embedded in the robot structure rather then being mounted only on the end/effectors to estimate the joint torques, as explained in \cite{Fumagalli2012}.
    } \label{fig:sensors}
  \end{minipage}
  \begin{minipage}[c]{0.16\textwidth}
    \includegraphics[width=\textwidth]{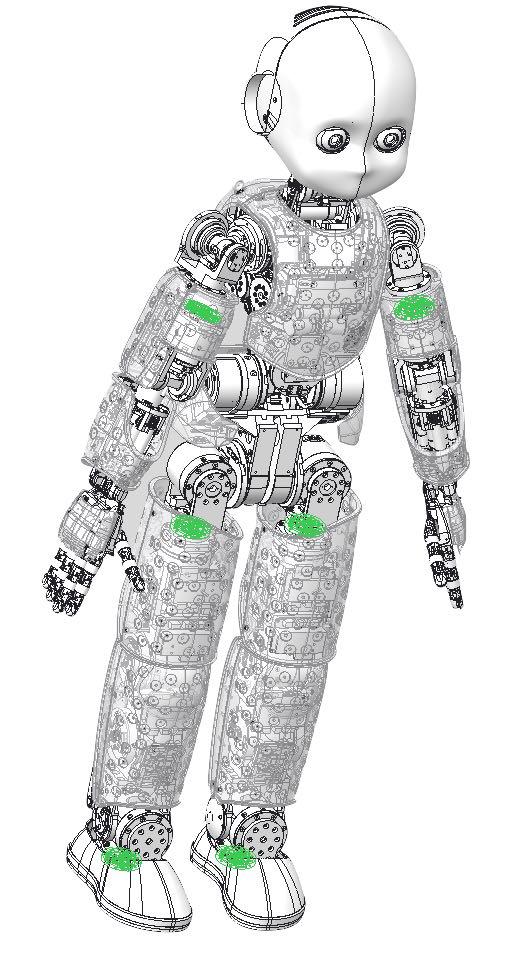}
  \end{minipage}\hfill
  \end{figure}

\section{PROBLEM STATEMENT}
\subsection{Notation}

The following notation is used throughout the paper.
\begin{itemize}
 \item The Euclidean norm of either a vector or a matrix of real numbers is denoted by $\left\| \cdot \right\|^2$.
 \item Given a series of vectors $u_i \in \mathbb{R}^n$, $\mu_{u} \in \mathbb{R}^n$ is the mean value of the series calculated as $\frac{1}{N}\sum_{i = 1}^N u_i$.  
\end{itemize}
\subsection{F/T Sensors Calibration}
\label{sec:problem}
 The strain gauge technology bases its measurements in the changes of the resistance according to small deformations of the material. The sensor is designed such that the resulting deformation in the sensors structure are inside the linear section of the material for the specified range. Because of this, a linear relationship between deformation and forces is assumed.
We assume that the model of the sensor is linear and has the following form:
\begin{equation}
w =Cr+o \label{eq:lin}
\end{equation}
where $w \in \mathbb{R}^6$  are the wrenches, $C \in \mathbb{R}^{6 \times 6}$ is the calibration matrix, $r \in \mathbb{R}^6$ are the raw measurements and $o \in \mathbb{R}^6$ is the offset.

Both the calibration matrix $C$ and the offset $o$ are typically unknown and need to be estimated.
Assuming that for a series of raw measurements $r_i$, we have the corresponding wrench applied on the sensor $w_i$, we can cast the problem of finding the calibration matrix and the offset as a multiple line regression using least squares fitting technique. The calibration matrix estimation can be considered as six different problems in which each row is a separate problem with six independent variables as input and one dependent variable output. For the sake of simplicity we solve all six axis at once. Thus the problem is stated as follows:
\begin{equation}
\argmin_{C , o}\ ~ \frac{1}{N}\sum_{i = 1}^N \left\|w _i - Cr_i - o\right\|^2 \\
\label{eq:prob}
\end{equation} 
Where $N$ is the number of data samples in the dataset. In classical calibration matrix estimation algorithms, the input data $(r_i, w_i) ,\  i = 1\dots N$ are obtained by applying a set of known masses in known locations with the sensor mounted on the workbench. For this reason we will refer to this kind of calibration as \emph{Workbench} calibration. 

As discussed in \cite{InSituAcc}, it is typically preferred to estimate offset separately from the calibration matrix, as the offset can typically vary across different experiments due to temperature drift, so the offset is removed from the raw measurements separately, and the calibration problem is reduced to: 
\begin{equation}
\argmin_{C}\ ~ \frac{1}{N}\sum_{i = 1}^N \left\|w _i - Cr_i\right\|^2 \\
\label{eq:probNoOffset}.
\end{equation} 

\section{MODEL BASED IN SITU CALIBRATION METHOD}
\label{sec:method}

Once a sensor is mounted in a complex structure such as humanoid robot, its calibration matrix may change due to the internal deformation caused by the mounting screws and other mounting deformations \cite{InSituAcc}. 
For this reason we need to \emph{recalibrate} the F/T sensor using a set of \emph{in situ} samples $(r_i, w_i) ,\  i = 1\dots N$ obtained directly on the robot. 

If it is known that no external wrench is acting on the limb on which the F/T sensor is mounted, then the expected wrench applied on the sensor can be computed using limb model and the instantaneous limb joints position, velocity and acceleration  \cite{Fumagalli2012}.  

\subsection{Centralized offset removal from training in situ datasets}
Once the \emph{in situ} calibration data $(r_i, w_i) ,\  i = 1\dots N$ are available, we need to get rid of the offset, even before estimating the calibration matrix. 

In this section, we propose a method to obtain a problem in the form \eqref{eq:probNoOffset}, without the need of computing the offset $o$. 

We centralize the data since for a problem of the form of 
\ref{eq:prob}, the solution for the optimal calibration matrix $C^{*}$ and optimal offset $o^{*}$ is given by 
\begin{equation}
o^{*}=\mu_w -C^{*}\mu_r.
\end{equation} 

So the form of the problem becomes independent of the offset and can be reduced to:
\begin{align}
\hat{w}_i&=w _i-\mu_w, \qquad  \label{eq:centralizedRaw}
\hat{r}_i= r_i-\mu_r, \\
\label{eq:centralizedCalibration}
 \argmin_{C \in \mathbb{R}^{6\times6}}\ ~ &\frac{1}{N}\sum_{i = 1}^N \left\|\hat{w}_{i} - C\hat{r}_i \right\|^2 .%
\end{align}
Where $\mu_w \in \mathbb{R}^6$ is the vector of the mean of the wrenches, $\mu_r \in \mathbb{R}^6$ the vector of the mean of the inputs, $\hat{w} _{i} \in \mathbb{R}^6$ and $\hat{r}_i \in \mathbb{R}^6$ are the \emph{centralized} data. 
Note that even if in \eqref{eq:centralizedCalibration} we did not removed explicitly the offset, the resulting optimization problem has the same form of \eqref{eq:probNoOffset}, and so for the calibration point of view the proposed algorithm is equivalent to offset removal. A proof for this statement is provided in the appendix \ref{appTheorem}.

\subsection{Model based in situ calibration matrix estimation}
Considering the linear model in \eqref{eq:lin}, a least squares technique is used for performing the linear regression. Assuming the calibration performed on the sensor was correct, we assume the new calibration matrix must not be very different from the matrix obtained using Workbench calibration. To enforce this assumption, we introduce a regularization term to penalize the difference with respect to the Workbench matrix. The new calibration matrix is obtained through the following optimization problem :
\begin{equation}
\label{eq:modelInSitu}
 C^{*} =   \argmin_{C \in \mathbb{R}^{6\times6}}\ ~ \frac{1}{N}\sum_{i = 1}^N \left\|\hat{w} _{i} - C\hat{r_i} \right\|^2 \\+\lambda\left\|C-C_w\right\|^2
\end{equation}
Where $C_w \in \mathbb{R}^{6 \times 6}$ is the \emph{Workbench} calibration matrix provided by the manufacturer, $\lambda$ is used to decide how much to penalize the regularization term and  N is the number of data points in the dataset. The regularization is added in order to try to keep the calibration matrix as close to the Workbench but with an improved performance once the sensor is already mounted on the system.

\section{EXPERIMENTS} 
Experiments have been performed on the 53 DOF robot iCub. Six custom-made six axes F/T sensors \cite{IITsensors}, one per ankle, leg and arm, are placed as shown in
Fig. \ref{fig:sensors}.
\begin{figure*}[t!]
    \centering
    \begin{subfigure}[t]{0.15\textwidth}
        \centering
        \includegraphics[width=\textwidth]{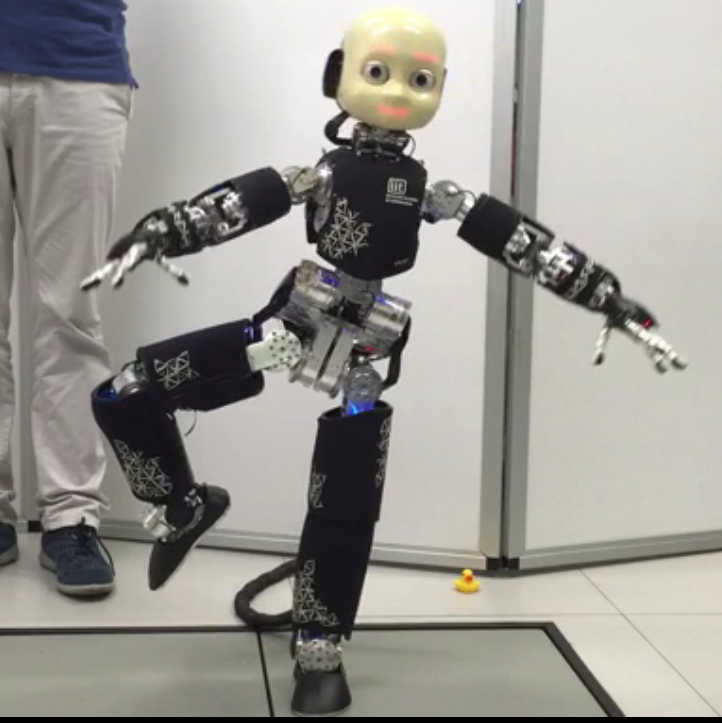}
        \caption{Bending knee on left leg support}
    \end{subfigure}%
      ~ 
    \begin{subfigure}[t]{0.15\textwidth}
        \centering
        \includegraphics[width=\textwidth]{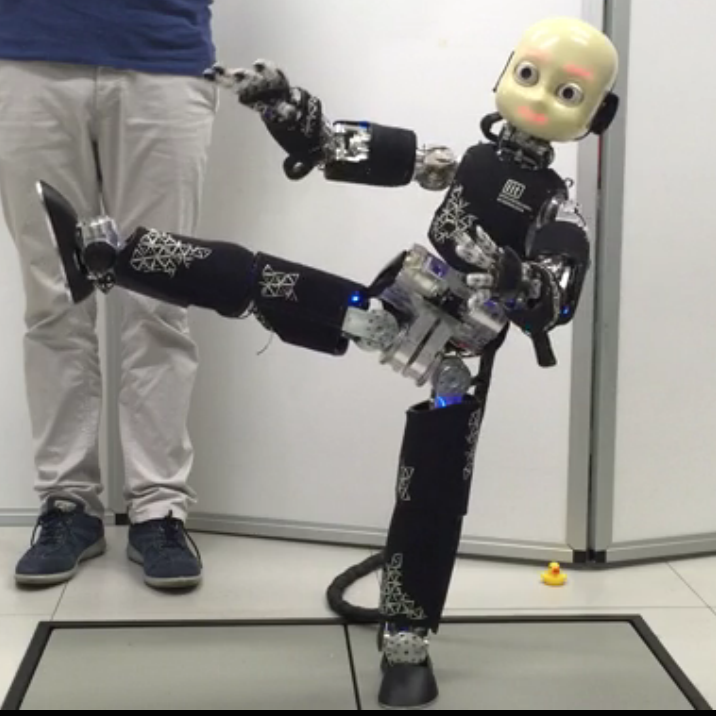}
        \caption{Stretched leg on left leg support}
    \end{subfigure}
    ~ 
    \begin{subfigure}[t]{0.15\textwidth}
        \centering
        \includegraphics[width=\textwidth]{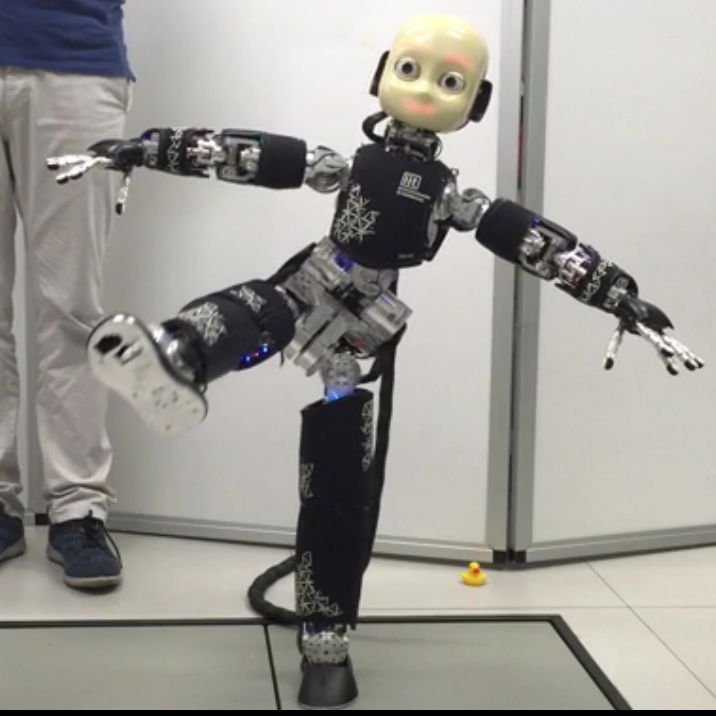}
        \caption{Moving back to front and sideways}
    \end{subfigure}
      ~ 
    \begin{subfigure}[t]{0.15\textwidth}
        \centering
        \includegraphics[width=\textwidth]{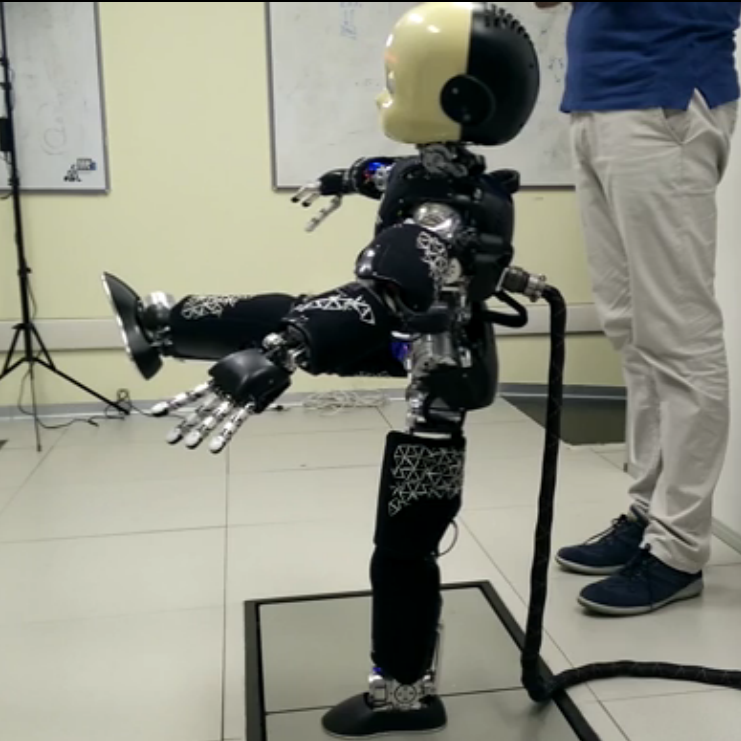}
        \caption{Stretching to the front, side view}
    \end{subfigure}
     ~ 
    \begin{subfigure}[t]{0.15\textwidth}
        \centering
        \includegraphics[width=\textwidth]{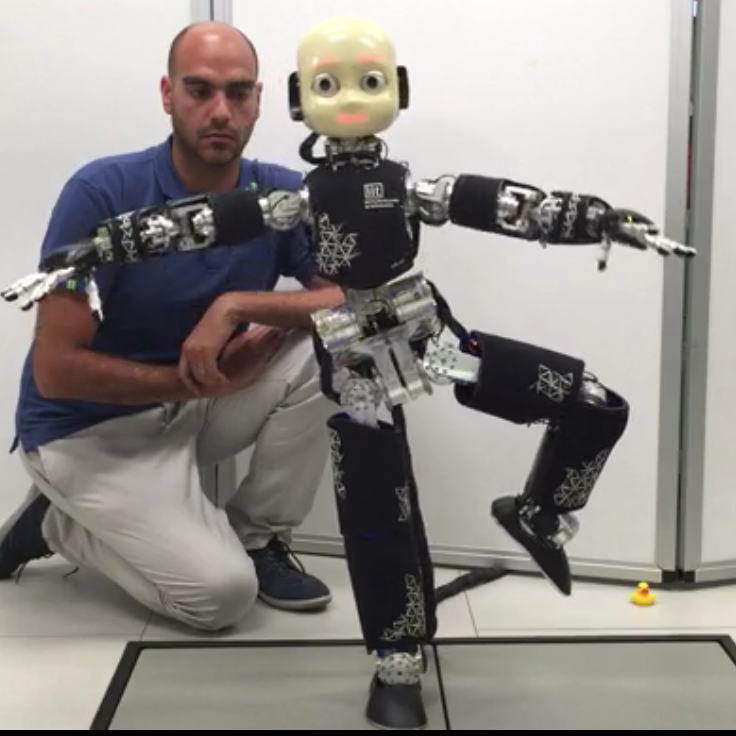}
        \caption{Bending knee on right leg support}
    \end{subfigure}
     ~ 
    \begin{subfigure}[t]{0.15\textwidth}
        \centering
        \includegraphics[width=\textwidth]{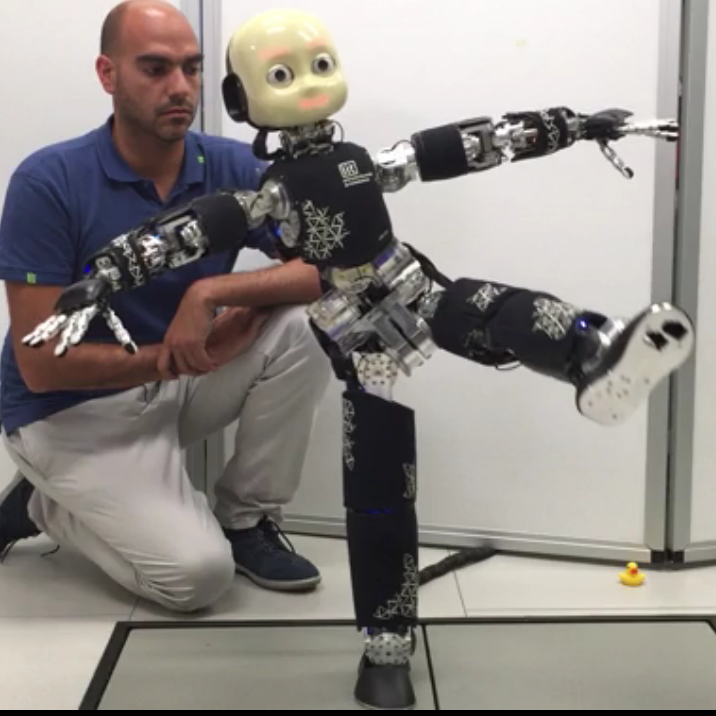}
        \caption{Streched leg on right leg support}
    \end{subfigure}
       \caption{Images from the extended balancing demo with contact switching}
       \label{fig:demo}
\end{figure*}

\label{sec:experiments}

\subsection{Remark on offset removal}
While the higher unit resistance and sensitivity of silicon based gauges are definite advantages, their greater sensitivity to temperature variations and tendency to drift are disadvantages in comparison to metallic foil sensors \cite{Transactions}.
This plus the effect of hysteresis may result in a different offset value for each experiment, but it is assumed the offset remains constant during the expermient due to the small time frame of each experiment. The effects of drift and hysteresis are currently not considered in the model. This also mean that the offset estimation should be done before every experiment, but this is trivial once the calibration matrix is known.

In the experimental settings, two methods were compared to remove the offset from the estimation problem. The first one is the \emph{centralized offset removal} introduced in this paper, while the second method is the \emph{in situ offset estimation}  proposed in \cite{InSituAcc}, where the accelerometers measurements are simulated using the kinematic model of the robot. This method has the advantage that the \emph{real} offset is directly calculated in the raw data and is independent of the estimated calibration matrix. It assumes a constant center of mass which prevents the use of this method in all datasets since this assumption is not always enforced.

In both cases we end up with a modified version of the raw data in which the effect of the offset is removed. With a little abuse of notation we have:
\begin{equation}
\hat{r}_i= \begin{cases} 
      r_i-o_r & \text{in situ offset estimation} \\
      r_i-\mu_r & \text{centralized offset removal}
   \end{cases}
\end{equation}
\begin{equation}
\hat{w}_i = \begin{cases} 
      w_i & \text{in situ offset estimation} \\
      w_i-\mu_w  & \text{centralized offset removal}
   \end{cases}
\end{equation}
Where $\hat{r}_i$ and $\hat{w}_i$ are the data used to solve the model based in situ calibration problem \eqref{eq:modelInSitu}.

\subsection{Training Datasets} 
We have 3 types of datasets:
  \begin{itemize}
      \item \textbf{Grid}: moving the legs in a grid pattern on a fixed pole. The contact is on the waist of the robot.
      \item \textbf{Balancing}: doing a one foot balancing demo. The external contact is the support foot. A video of this demo can be found in \cite{video}.
      \item \textbf{Extended Balancing}: doing an extended one foot balancing demo with more widespread leg movements. The contact is on one support foot.
  \end{itemize}
  Both the balancing demo and the extended balancing involve more general movements like flexing the legs, reason for which the offset can only be removed using the \emph{centralized offset removal} method.
  
The following assumptions are valid for all datasets:
    \begin{itemize}
        \item There is only one contact point where external force is applied and its location is known.
        \item Rigid body inertial parameters are known.
        \item Relation between raw measurements and F/T values is a linear affine transformation.
        \item The offset does not change during the experiment, although it may be different in the other experiments.
        \item The effects of drift and hysteresis are negligible due to the short time range of the experiments.
    \end{itemize}

The wrenches used as reference are estimated through the model and kinematic measurements using the methods described in \cite{Fumagalli2012}.
\subsection{ Validation procedures}
For comparing and validating the results 7 different datasets where used: one balancing experiment, two extended balancing at slow speed, two extended balancing with a faster movement and two grid pattern experiments one with a wider joint range than the other. Most of the datasets were taken on different days. All of them were taken on the same robot with the same sensors, and the sensors were never unmounted from the robot in the time between two experiments. 

Two different validation procedures were done to compare the resulting calibration matrices. In the first procedure, we simply test the calibration matrices obtained on each dataset with $\lambda=0$ against all others including the original calibration matrix provided by the manufacturer, which will be called hence forth as the \emph{Workbench} matrix. This is in order to determine if one experiment was more representative and useful than the others.\\
 For the second validation procedure different calibration matrices are generated with varying values of $\lambda$ and then are tested on a dataset performing movements similar to the actual use of the robot, in this case an extended balancing dataset with contact switching. They are compared through an estimation of the external forces in the section between the ankle and the hip. This computation is done through the algorithm proposed in \cite{Fumagalli2012}. The estimation considers the gravity and, in the ideal case, it should show a value of 0, since during the experiments no external force was applied to the robot in that subchain. The experiment in which it was validated starts with the robot hanging in the air. Then performing the extended balancing demo on both feet one after the other as shown in fig. \ref{fig:demo}. This is to include many of the behaviours expected form a general use of the robot and see the benefits of the \textit{In situ} calibration. The important restriction kept during the experiment is that the only external force acting on the robot was gravity.\\
In the first method the best calibration matrix is the one in which the error with respect to the estimated wrenches is lower. To avoid problems with comparing different units, $N$ and $Nm$, we do the comparison on the error percentage calculated in the following way: 
\begin{equation*}
    e_{d\%}=\frac{w_{ref}-C_dr_{ref}}{w_{ref}-C_wr_{ref}}
\end{equation*}
Where $d$ is the dataset in which the calibration matrix was calculated, $ref$ is the dataset in which the calibration matrix is being tested, $C_d \in \mathbb{R}^{6 \times 6}$ is the calibration matrix calculated on the $d$ dataset and $e_{d\%}$ is the error percentage of $C_d$.
In this case if an experiment performs worse than the Workbench matrix it is automatically discarded for the second validation procedure. 


\section{RESULTS}
\label{sec:results}

\subsection{First validation procedure}
\rowcolors{1}{lightgray}{white}
 Table \ref{tab:1} summarizes the errors of all axis of all datasets against all datasets. The rows represent the dataset in which it was tested and the columns the calibration matrix that was used to calculate the error. It can be seen that the calibration matrix obtained from the simple balancing experiment performs worse than those of the Workbench matrix. This is confirmed in figure \ref{fig:all1valid}, which as an example of the error of the different calibration matrices on a dataset. All the other calibration matrices give results relatively close to each other and are indeed way better than the results of the Workbench matrix.
\begin{table*}[ht]
\centering
\caption{Error percentage between calibration matrix and estimated wrenches ($ e_{d\%}$) over all datasets.}
\resizebox{0.9\textwidth}{!}{\begin{minipage}{\textwidth}
\centering 
\begin{tabular}{|c|c|c|c|c|c|c|c|c|}
\hline 
 Dataset ($ref$) & Workbench & Balancing & gridMin30 & gridMin45 & ExtBal1 & ExtBal2 & fastExtBal & fastExtBal2 \\ 
\hline 
Balancing & 100 & 17.29 & 110.38 & 103.37 & 55.09 & 57.53 & 0.4977  & 0.5813 \\ 
\hline 
gridMin30 & 100 & 464.82 & 13.19 & 18.29 & 25.13 & 23.69 & 0.3311 & 0.2426 \\ 
\hline 
gridMin45 & 100 & 532.59 & 12.94 & 10.99 & 22.47 & 25.42 & 29.68 & 25.61 \\ 
\hline 
ExtBal1 & 100 & 102.50 & 39.84 & 30.20 & 25.88 & 23.48 & 34.70 & 25.81 \\ 
\hline 
ExtBal2 & 100 & 128.64 & 33.03 & 28.02 & 28.25 & 26.83 & 35.26 & 25.76 \\ 
\hline 
fastExtBal & 100 & 107.44 & 43.51 & 31.65 & 28.00 & 26.14 & 36.57 & 26.55 \\ 
\hline 
fastExtBal2 & 100 & 118.52 & 33.18 & 29.80 & 27.87 & 24.45 & 34.94 & 27.34 \\ 
\hline 
\end{tabular}
\end{minipage}}
\label{tab:1}
\end{table*}

\begin{figure}
        \centering
        \includegraphics[width=.45\textwidth]{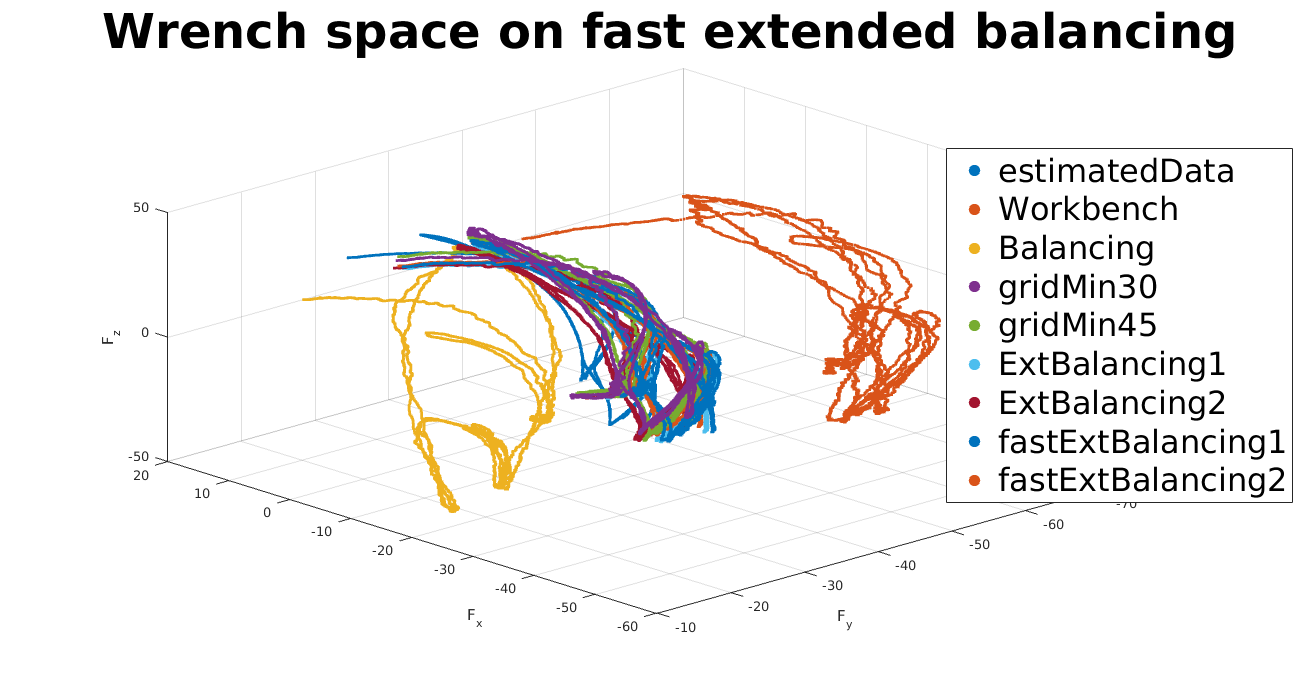}
        \caption{3D force comparison among the calibration matrices trained on each dataset against the model estimated forces on the fastExtBal1 dataset. }
        \label{fig:all1valid}
\end{figure}
 Using this procedure the calibration matrix with better results on general is the $2_{nd}$ extended balancing dataset with the code name of ``ExtBal2''.\\
 Since each axis can be seen as an independent problem we also considered the calibration matrix which has better results in each axis. The results are shown on Table \ref{tab:axis1}.
 
 \begin{table*}
 \centering
 \caption{Mean value of error on the  external force (N) estimated on the extended balancing dataset using the different calibration matrices with different $\lambda$ values}
 \resizebox{\textwidth}{!}{\begin{minipage}{\textwidth}
 \centering
\begin{tabular}{|c|c|c|c|c|c|c|c|c|c|} 
 \hline  
 Dataset ($ref$) & 0 & 0.5 & 1 & 1.5 & 2 & 4 & 6 & 8 & 10  \\ 
 \hline  
 Workbench  & 31.3111 & - & - & - & - & - & - & - &\\  
 \hline  
 ExtBal1 & -  & 30.7219  & 30.7553  & 30.7218  & 30.7498  & 30.7285  & 30.7670  & 30.7593  & 30.7847 \\  
 \hline  
 ExtBal2 & - & 33.5694  & 33.5650  & 33.5620  & 33.5504  & 33.5609  & 33.5823  & 33.5533  & 33.5698 \\ 
 \hline  
 fastExtBal2 & -  & 33.8062  & 33.8423  & 33.8398  & 33.8469  & 33.8275  & 33.8646  & 33.8530  & 33.8441 \\ 
 \hline  
 gridMin30 & - & 23.5073  & 23.5635  & 23.6338  & 23.7039   & 23.9476  & 24.1988  & 24.4084  & 24.6439 \\ 
 \hline  
 gridMin45 & - & 24.2435   & 24.2729  & 24.3017   & 24.3390   & 24.4010   & 24.4761  & 24.5821  & 24.6940 \\ 
 \hline  
 \end{tabular}
 \end{minipage}}
 \label{tab:all2}
 \end{table*}

 
 The results from the first validation procedure highlight how the calibration matrix trained on the normal balancing dataset was found to perform consistently worse then the \emph{Workbench} calibration matrix, and for this reason it was not included in the second validation procedure. 
 
\subsection{Second validation procedure}
The $\lambda$ values used for the second validation procedure are [0.5 1 1.5 2 4 6 8 10]. 
41 different calibration matrices where tested. Eight calibration matrices where generated for each valid dataset except the one in which they were tested. The results can be summarized in the table \ref{tab:all2}.

 For this procedure the calibration matrix with better results on general is the fixed pole dataset with the code name of ``gridMin30''. The difference between the magnitude of the forces at each axis between the Workbench and the ``gridMin30'' can be seen in figure \ref{fig:WvsB}. The mayor improvements are on the $f_y$ and $f_x$ axis.
 
 Since each axis can be seen as an independent problem, it is possible to see a further improvement if we take the best for each axis separately, which we will call "mixed". This comparison can be seen in figure \ref{fig:BvsF}. The results are shown in Table \ref{tab:axis1}. 
 \begin{figure}
        \centering
        \includegraphics[width=.45\textwidth]{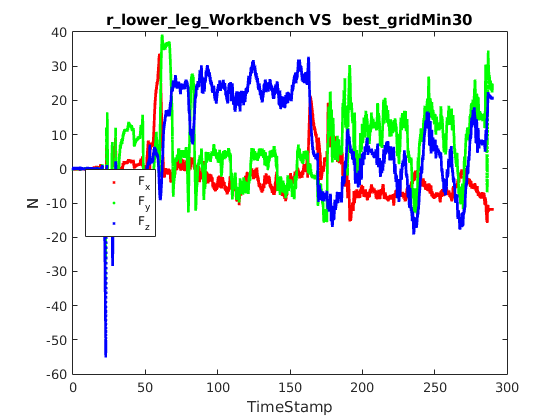}
        \caption{Difference between forces obtained with the Workbench calibration and the best \emph{in situ} calibration matrix}
        \label{fig:WvsB}
\end{figure}

 
\begin{table}
 \centering
 \caption{Best calibration matrix for each axis}
 \begin{tabular}{|c|c|c|c|c|} 
  \hline  
  & \multicolumn{2}{|c|}{$1_{st}$ Validation}   & \multicolumn{2}{|c|}{$ 2_{nd}$ Validation}  \\ 
 \hline  
 Axis & Dataset & Error $\%$ & Dataset & Mean $N$ / $Nm$ \\ 
 \hline  
 $f_x$  & ExtBal2 & 28.43 & ExtBal1$_{\lambda10} $ & 8.13\\  
 \hline  
 $f_y$ & fastExtBal & 14.30 & gridMin30 & 14.83 \\  
 \hline  
 $f_z$ & ExtBal1& 56.85  & gridMin45 & 13.46  \\ 
 \hline  
 $\tau_x$ & ExtBal2  & 54.62 & fastExtBal2$_{\lambda1.5}$  & 4.60   \\ 
 \hline  
 $\tau_y$ & Workbench & 100 & Workbench & 0.58  \\ 
 \hline  
 $\tau_z$ & Workbench & 100 & fastExtBal2$_{\lambda10 }$ & 1.88  \\ 
 \hline  
 \end{tabular}
   \label{tab:axis1}
 \end{table}
 
\begin{figure}
        \centering
        \includegraphics[width=.45\textwidth]{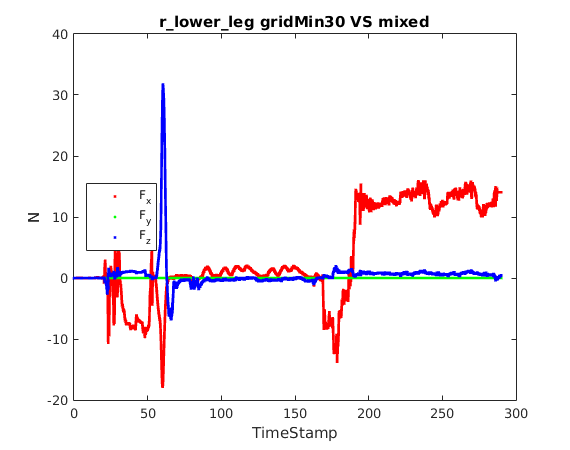}
        \caption{Difference between the forces obtained using the best new calibration matrix and the mixed matrix}
        \label{fig:BvsF}
\end{figure}

The improvement in the measurements was observed when the movement of the robot used to do the \textit{in situ} calibration spans a wide amount of the operational space. In the case of the first validation procedure we where also able to observe that when looking at the axis separately in the case of the torques the Workbench matrix still out performs the others.
It also justifies our choice of the regularization parameter to penalize the difference with respect to the Workbench. In the second validation procedure, it is possible to observe that, removing the offset as suggested in \cite{InSituAcc} gives calibration matrices that outperform the others in a more general scenario, something that was not reflected during the first validation procedure. This implies that a simple dataset with the robot on the pole moving the legs around is enough to perform a good enough calibration that can outperform datasets obtained using more complex scenarios. It should also be noted that difference in the magnitude of the force obtained with the new calibration matrix is around  25\% better than the Workbench. Since this calibration depends on the calibration matrix of two different sensors, the one in the ankle and the one in the hip, part of the magnitude of the force seen in this validation may be due to the fact that the one from the ankle has not been optimized. We expect better results once that sensor is also calibrated \emph{in situ}.

It can be seen that there is more than one optimal $ \lambda $ value depending on the axis. This confirms that the axis are indeed independent problems. Thus, it is possible to use a different solution per axis to obtain an even better final calibration matrix for the F/T sensor.

\section{CONCLUSIONS}
\label{sec:conclusions}
In this paper we introduced a new \emph{in situ} calibration technique for six axis F/T sensors that exploit the a-priori knowledge of the inertial parameters model of the robot. We show how the use of such a technique improved the quality of force measurements in the case of the legs of the iCub humanoid robot. 
Regarding future work we plan to insert the new calibration matrix in the robot and measure the improvements in the performance of the whole body balancing controller \cite{nava2016} as well as finding a set of optimal minimal poses or trajectories to do the calibration.

\section*{APPENDIX}
\label{appTheorem}
\begin{theorem}
If $C^*, o^*$ are the solutions to the calibration problem \eqref{eq:prob}:
\begin{equation}
C^* , o^* = \argmin_{C , o}\ ~ \frac{1}{N}\sum_{i = 1}^N \left\|w _i - Cr_i - o\right\|^2 .
\end{equation}

We have that:
\begin{IEEEeqnarray}{rCl}
C^* &=& \argmin_{C}\ ~ \frac{1}{N}\sum_{i = 1}^N \left\|\hat{w}_i - C\hat{r}_i\right\|^2 , \\
o^* &=& \mu_{w} - C^* \mu_r .
\end{IEEEeqnarray}
\end{theorem}

\begin{proof}
Using the definitions of $\hat{w}_i$ and $\hat{r}_i$ we can write the cost function in \eqref{eq:prob} as:
\begin{align*}
\frac{1}{N}\sum_{i = 1}^N \left\|\hat{w}_i - C\hat{r}_i + \mu_w - C \mu_r - o\right\|^2 = \\
= 
\frac{1}{N}\sum_{i = 1}^N \left\|\hat{w}_i - C\hat{r}_i  \right\|^2 
+ 
\frac{1}{N}\sum_{i = 1}^N \left\|\mu_w - C \mu_r - o\right\|^2 
+ \\ +
\frac{2}{N}\sum_{i = 1}^N (\hat{w}_i - C\hat{r}_i)^\top(\mu_w - C \mu_r - o) .
\end{align*}

As $\sum_{i = 1}^N \hat{w}_i = 0$ and $\sum_{i = 1}^N \hat{r}_i = 0$ from their definition \eqref{eq:centralizedRaw} we get that the third term of the is always equal to zero, and so we have that the calibration problem reduces to: 
\begin{align*}
C^* , o^* = \argmin_{C , o}\ ~ \big(& \frac{1}{N}\sum_{i = 1}^N \left\|\hat{w}_i - C\hat{r}_i  \right\|^2 
+ \\ &+
 \left\|\mu_w - C \mu_r - o\right\|^2 \big)
\end{align*}
Noting that the minimum of the second term is always $0$ for $o = \mu_w - C \mu_r, \hspace{0.3em} \forall C$, we prove the theorem.
\end{proof}

\bibliographystyle{IEEEtran}
\bibliography{Literature/bib}

\addtolength{\textheight}{-12cm}

\end{document}